\documentclass[11pt]{article}
\usepackage{microtype}
\usepackage{graphicx}
\usepackage{caption,subcaption}
\usepackage{amsmath,amsbsy,amsfonts,amssymb,amsthm,bm}
\usepackage{algorithm,algorithmic,mathtools}
\usepackage{color,cases,multirow}
\usepackage[margin=1.2in]{geometry} 
\usepackage[colorlinks=True, citecolor=blue]{hyperref}
\usepackage{todonotes}
\usepackage{authblk}

\usepackage{amsmath,amsbsy,amsfonts,amssymb,amsthm,bm,color}
\newtheorem{theorem}{Theorem}

\newtheorem{proposition}[theorem]{Proposition}

\newtheorem{assumption}{Assumption}

\theoremstyle{definition}

\newcommand{\bR}{\mathbb{R}}
\newcommand{\bE}{\mathbb{E}}
\newcommand{\bS}{\mathbb{S}}
\newcommand{\bN}{\mathbb{N}}
\newcommand{\bx}{\mathbf{x}}

\newcommand{\bv}{\mathbf{v}}

\newcommand{\bb}{\mathbf{b}}
\newcommand{\barb}{\bar{\mathbf{b}}}

\newcommand{\cP}{\mathcal{P}}
\newcommand{\cM}{\mathcal{M}}
\newcommand{\bP}{\mathbb{P}}

\newcommand{\var}{\textrm{Var}}

\newcommand{\divg}{\textrm{div}}
\newcommand{\BN}{\textrm{BN}}
\newcommand{\dist}{\textrm{dist}}

\newcommand{\btheta}{\mathbf{\theta}}

\title{A Riemannian Mean Field Formulation for Two-layer Neural Networks with Batch Normalization}
\author{Chao Ma, Lexing Ying}
\date{Department of Mathematics\\
	Stanford University}

\begin{document}

\maketitle

\begin{abstract}
The training dynamics of two-layer neural networks with batch normalization (BN) is studied. It is written as the training dynamics of a neural network without BN on a Riemannian manifold. Therefore, we identify BN's effect of changing the metric in the parameter space. Later, the infinite-width limit of the two-layer neural networks with BN is considered, and a mean-field formulation is derived for the training dynamics. The training dynamics of the mean-field formulation is shown to be the Wasserstein gradient flow on the manifold. Theoretical analysis are provided on the well-posedness and convergence of the Wasserstein gradient flow.
\end{abstract}

\section{Introduction}
Batch normalization (BN)~\cite{ioffe2015batch} is a technique that greatly helps the training of deep neural networks. It is used in almost every neural network model in real-world applications. Given the benefit of batch normalization in practice, efforts are invested to explain the essential reason of its success. Existing works address this issue from many different perspectives, including covariate shift, landscape smoothing, length-direction decoupling, and learning rate adaptivity, etc~\cite{ioffe2015batch,santurkar2018does,kohler2018towards,arora2018theoretical}, while no explanation is generally accepted to have ``solved'' the problem.

In this work, we provide a metric point of view to the effect of BN on the optimization dynamics of neural networks. We show that, the gradient descent dynamics of a two-layer neural network with batch normalization can be written as the gradient descent of a two-layer neural network without BN (with the same width) on a Riemannian manifold. Hence, batch normalization changes the metric of the parameter space. We explicitly write down the manifold and the metric of the Riemannian manifold. Next, we consider the infinite-width limit, i.e. the continuous formulation, of the model, and derive a mean-field formulation for the training dynamics of the BN model by gradient flow. The mean-field dynamics is naturally a Wasserstein gradient flow on the Riemannian manifold. Adopting techniques from previous works~\cite{chizat2018global}, under appropriate conditions we show the existence of the solution and the global optimality of convergent solutions. 

Finally, with the Riemannian manifold understanding of BN's effect, we identify several potential benefits of BN on the training of neural networks. First, models with batch normalization can adjust the speed of neurons according to the alignment of the neurons' directions with the data distribution, which helps neurons find significant directions quickly. Second, BN models assign different speeds to neurons with different magnitudes. Hence, BN models with diverse neuron length can explore an ensemble of learning rates. Finally, we also identify a first-step amplification effect that guarantees the well-behavedness of the BN model under discrete time gradient descent dynamics.

As a summary, our main contributions are as follows:
\begin{enumerate}
\item We show that the gradient flow dynamics of two-layer neural networks with BN is equivalent with the gradient flow dynamics of vanilla two-layer neural networks on a Riemannian manifold. Hence, batch normalization changes the metric of the parameter space. 

\item In the infinite-width limit, we derive a mean-field formulation for the training dynamics of BN models. The mean-field formulation is a Wasserstein gradient flow on the Riemannian manifold. We analyze the existence and convergence of the dynamics. 

\item With the Riemannian manifold understanding, we identify and discuss several special features of the training dynamics introduced by batch normalization.
\end{enumerate}

\section{Gradient flow on the Riemannian manifold}\label{sec:manifold}

\subsection{Preliminaries} 

In this paper, we consider two-layer neural networks with batch normalization before the activation function:
\begin{equation}\label{eqn:2lnn}
    f(\bx,\btheta) = \frac{1}{m}\sum\limits_{k=1}^m a_k \sigma(\BN(\bb_k^T\bx)),
\end{equation}
where we suppose $\bx, \bb_k\in\bR^{d}$ and $a_k\in\bR$, and $\btheta$ is a parameter vector containing all the entries in $(a_k,\bb_k)_{k=1}^m$. Let $\mu$ be the probability distribution from which the data $\bx$ are sampled. Then, for any $\bb\in\bR^d$ we consider the population batch normalization $\BN(\bb^T\bx)$ defined as 
\begin{equation}\label{eqn:bn}
    \BN(\bb^T\bx) = \frac{\bb^T\bx - \bE_{\bx\sim\mu}\bb^T\bx}{\sqrt{\var_{\bx\sim\mu}[\bb^T\bx]}},
\end{equation}
For the ease of analysis, we make the following assumptions on the data distribution $\mu$.
\begin{assumption}\label{assump:1}
Let $\Sigma=\bE_{\bx\sim\mu}\bx\bx^T$. Assume $\bE_{\bx\sim\mu}\bx=0$ and $\Sigma\succ0$.
\end{assumption}
The assumption above is reasonable considering that data are usually normalized before being fed into the neural networks. Then, \eqref{eqn:bn} can be written as 
\begin{equation}\label{eqn:bn2}
    \BN(\bb^T\bx) = \frac{\bb^T\bx}{\sqrt{\bb^T\Sigma\bb^T}}.
\end{equation}
Let $\|\bb\|_\Sigma = \sqrt{\bb^T\Sigma\bb^T}$ and $\bar{\bb}=\bb/\|\bb\|_\Sigma$, we have $\BN(\bb^T\bx)=\barb^T\bx$, and hence we can write~\eqref{eqn:2lnn} as
\begin{equation}\label{eqn:2lnn_2}
    f(\bx,\btheta) = \frac{1}{m}\sum\limits_{k=1}^m a_k \sigma(\bar{\bb}_k^T\bx).
\end{equation}

Next, consider a supervised learning problem with input-target pairs $(\bx,y)$ sampled i.i.d. from probability distribution $\bP$. Note that the marginal distribution of $\bP$ on $\bx$ is $\mu$. Let $l(\cdot,\cdot)$ be a risk function which is twice differentiable. Then, using the BN model~\eqref{eqn:2lnn_2} to learn the supervised learning problem asks to minimize the following loss function:
\begin{equation}\label{eqn:loss_finite}
    L(\btheta) = \bE_{(\bx,y)\sim\bP}\  l\left(f(\bx,\btheta), y\right).
\end{equation}
The minimization is achieved by some optimization algorithms. In this paper, we study the gradient flow (the zero learning rate limit of the gradient descent algorithm), which is given by the following ODEs:
\begin{align}
\dot{a}_k &= -\frac{\partial L(\theta)}{\partial a_k} = \frac{1}{m}\bE l'(f(\bx,\btheta), y)\sigma(\barb_k^T\bx), \nonumber\\
\dot{\bb}_k &= -\frac{\partial L(\theta)}{\partial \bb_k} = \frac{1}{m}\bE l'(f(\bx,\btheta), y)a\sigma'(\bar\bb_k^T\bx)\frac{1}{\|\bb_k\|_\Sigma}\left(I-\frac{\Sigma\bb_k\bb_k^T}{\|\bb_k\|^2_\Sigma}\right)\bx. \label{eqn:gf_bn}
\end{align}
In~\eqref{eqn:gf_bn} and below, we drop the subscript of the expectation and use $\bE$ to denote the expectation over $\bP$.
% \LY{maybe it will be useful to say we will drop the subscript in $\bE_{(\bx,y)\sim\bP}$ from now on.}
With the dynamics above, it is easy to show that $\frac{d}{dt}\|\bb_k\|^2=2\bb_k^T\dot{\bb_k}=0$. Hence, the norm of $\bb_k$ does not change during training. By the formulation of batch normalization, $\|\bb_k\|$ does not influence the function implemented by the model, too. Therefore, without loss of generality, we can assume $\bb_k\in\bS^{d-1}$ for any $k=1,2,...,m$ and at any time, where $\bS^{d-1}$ is the unit sphere in $\bR^d$.

\subsection{Gradient flow on Riemannian manifold}
By Equation~\eqref{eqn:2lnn_2}, the BN model represents the same function as a vanilla two-layer neural networks with $(a_k,\barb_k)$ as parameters. Let $f_0(\bx,\btheta)$ be the vanilla model
\begin{equation*}
    f_0(\bx,\btheta) = \frac{1}{m}\sum\limits_{k=1}^m a_k\sigma(\bb_k\bx),
\end{equation*}
and let $\bar{\btheta}$ be the set of parameters containing $(a_k,\barb_k)_{k=1}^m$. Then, we have $f(\bx,\btheta) = f_0(\bx,\bar{\btheta})$. Yet, when training is considered, the dynamics of parameters are different. If we view the BN model as a vanilla model using the equivalence above, the dynamics of $\bar{\theta}$ is induced by the gradient flow for the BN model as described in~\eqref{eqn:gf_bn}. We have
\begin{align}
\dot{a}_k &= \frac{1}{m}\bE l'(f_0(\bx,\bar{\btheta}), y)\sigma(\barb_k^T\bx), \nonumber\\
\dot{\barb}_k &= -\frac{1}{m}\|\bar{\bb}_k\|^2(I-\bar{\bb}_k\bar{\bb}_k^T\Sigma)(I-\Sigma\bar{\bb}_k\bar{\bb}_k^T)\bE l'(f_0(\bx,\bar{\btheta}), y)a_k\sigma'(\barb_k^T\bx)\bx.\label{eqn:gf_bn_model}
\end{align}
Recall that the gradient flow for the vanilla non-BN model at $\bar{\btheta}$ is
\begin{align}
\dot{a}_k &= \frac{1}{m}\bE l'(f_0(\bx,\bar{\btheta}), y)\sigma(\barb_k^T\bx), \nonumber\\
\dot{\barb}_k &= -\frac{1}{m}\bE l'(f_0(\bx,\bar{\btheta}), y)a_k\sigma'(\barb_k^T\bx)\bx. \label{eqn:gf_nonbn_model}
\end{align}
We see that the two dynamics~\eqref{eqn:gf_bn_model} and~\eqref{eqn:gf_nonbn_model} differ at the term $\|\bar{\bb}_k\|^2(I-\bar{\bb}_k\bar{\bb}_k^T\Sigma)(I-\Sigma\bar{\bb}_k\bar{\bb}_k^T)$, which depends on the location of $\barb$ and the data covariance $\Sigma$. In the following, we show that this term appears if we consider the gradient flow of the non-BN model on a special Riemannian manifold. This means the GF dynamics of the BN model in the Euclidean space is equivalent with the GF dynamics of the non-BN model on a manifold. 

To see this, first note that for any $k=1,2,...,m$ we always have $\|\barb_k\|_\Sigma=1$. Let $\Omega=\{\bb: \|\bb\|_\Sigma=1\}$ and $\cM=\bR\times\Omega$. The matrix
\begin{equation*}
  G_{(a,\bb)} = \left[\begin{array}{cc}
      1 & 0  \\
      0 & G_{\bb}
    \end{array}\right]
  \;\text{with}\;
  G_{\bb}=\frac{1}{\|\bb\|^2}\left(I-\frac{\Sigma\bb\bb^T\Sigma}{\bb^T\Sigma^2\bb}\right)[(I-\bb\bb^T\Sigma)(I-\Sigma\bb\bb^T)]^\dagger\in\bR^{d\times d}
\end{equation*}
induces a Riemannian metric on $\cM$.

\begin{proposition}\label{prop:Riemann_mani}
Let $T_{(a,\bb)}\cM$ be the tangent space of $\cM$ at $(a,\bb)$. For any $(a,\bb)$, define function $h:(T_{(a,\bb)}\cM)^2\rightarrow\bR:\ h(\alpha,\beta) = \alpha^T G_{(a,\bb)}\beta$, where $\alpha$ and $\beta$ are treated as vectors in $\bR^{d+1}$. Then, $h$ is an inner product on $T_{(a,\bb)}\cM$. 
\end{proposition}

\begin{proof}
Obviously, the metric along the direction of $a$ is the standard metric. Hence, we only need to show the results for $\bb$. Let $T_\bb\Omega$ be the tangent space of $\Omega$ at $\bb$. With an abuse of notation, let $\alpha,\beta\in T_\bb\Omega$. Viewing $\alpha$ and $\beta$ as vectors in $\bR^d$, we have
\begin{equation*}
\alpha^T G_\bb \beta = \frac{1}{\|\bb\|^2}\alpha^T\left(I-\frac{\Sigma\bb\bb^T\Sigma}{\bb^T\Sigma^2\bb}\right)[(I-\bb\bb^T\Sigma)(I-\Sigma\bb\bb^T)]^\dagger\beta. 
\end{equation*}
Since the tangent space can be written as $T_\bb\Omega=\{\alpha:\ \alpha^T\Sigma\bb=0\}$, we have
\begin{equation*}
    \alpha^T\left(I-\frac{\Sigma\bb\bb^T\Sigma}{\bb^T\Sigma^2\bb}\right)=\alpha.
\end{equation*}
Therefore,
\begin{equation*}
\alpha^T G_\bb \beta = \frac{1}{\|\bb\|^2}\alpha^T[(I-\bb\bb^T\Sigma)(I-\Sigma\bb\bb^T)]^\dagger\beta,
\end{equation*}
and we directly have $\alpha^T G_\bb \beta=\beta^T G_\bb \alpha$. 

Next, we show $h$ is positive definite. First, it is easy to show that $[(I-\bb\bb^T\Sigma)(I-\Sigma\bb\bb^T)]^\dagger$ is positive semi-definite. Second, since the pseudo-inverse of a symmetric matrix has the same $0$-eigenspace as the original matrix, if $\alpha^T[(I-\bb\bb^T\Sigma)(I-\Sigma\bb\bb^T)]^\dagger\alpha=0$ holds for some $\alpha$, then we must have $(I-\Sigma\bb\bb^T)\alpha=0$. Recall that $\alpha^T\Sigma\bb=0$, we then have
\begin{equation*}
    0=\alpha^T(I-\Sigma\bb\bb^T)\alpha = \|\alpha\|^2+\alpha^T\Sigma\bb\bb^T\alpha = \|\alpha\|^2.
\end{equation*}
Therefore, $h$ is strictly positive definite on $T_\bb\Omega$, which completes the proof. 
\end{proof}

In the following, $\cM$ is always assumed to be the Riemannian manifold endowed with the metric
$G_{(a,\bb)}$. The following theorem shows that, starting from the same initialization, the gradient
flow dynamics of BN model in the Euclidean
space equals to a corresponding non-BN model on $\cM$.

\begin{theorem}\label{thm:riemann_grad}
Let $f_0(\bx,\bar{\btheta}) = \frac{1}{m}\sum\limits_{k=1}^m a_k\sigma(\barb_k\bx)$. The manifold gradient of $L_0(\bar{\theta}):=\bE_{(\bx,y)\sim\bP}l(f_0(\bx,\bar{\btheta}),y)$ on $\cM$ with respect $\bar{\btheta}$ is
\begin{align}
\frac{\partial L_0(\bar{\btheta})}{\partial a_k} &= -\frac{1}{m}\bE l'(f_0(\bx,\bar{\btheta}), y)\sigma(\barb_k^T\bx), \nonumber \\
\frac{\partial L_0(\bar{\btheta})}{\partial \barb_k} &=-\frac{1}{m}\|\bar{\bb}_k\|^2(I-\bar{\bb}_k\bar{\bb}_k^T\Sigma)(I-\Sigma\bar{\bb}_k\bar{\bb}_k^T)\bE l'(f_0(\bx,\bar{\btheta}), y)a_k\sigma'(\barb_k^T\bx)\bx. \nonumber
\end{align}
\end{theorem}

\begin{proof}
Again, we focus on the $\barb$ part. For any $k=1,2,...,m$, the gradient of $L_0$ with respect to $\barb_k$ under Euclidean metric is 
\begin{equation*}
    \partial_{\barb_k} L_0(\bar{\btheta}) = \bE l'(f_0(\bx,\bar{\btheta}), y)a_k\sigma'(\bar{\bb}_k^T\bx)\bx.
\end{equation*}
Let $\cP_{\barb}$ be the orthogonal projection matrix onto $T_{\barb}\Omega$, and $\partial_{m,\barb_k} L_0(\bar{\btheta})$ be the manifold gradient. Then, we have
$$\cP_{\bar\bb} = I-\frac{\Sigma\barb\barb^T\Sigma}{\barb^T\Sigma^2\barb},$$
and the condition for the manifold gradient:
$$\bv^TG_{\barb_k}\partial_{m,\barb_k} L_0(\bar{\btheta}) = \bv^T P_{\barb_k}\partial_{\barb_k} L_0(\bar{\btheta}),$$
where $\bv\in\bR^d$ is an arbitrary vector. Therefore, we have
\begin{equation*}
    \partial_{m,\barb_k} L_0(\bar{\btheta}) = G_{\barb_k}^{\dagger}P_{\barb_k}\partial_{\barb_k} L_0(\bar{\btheta}).
\end{equation*}
Substituting 
$$G_{\barb_k} = \frac{1}{\|\barb_k\|^2}\left(I-\frac{\Sigma\barb_k\barb_k^T\Sigma}{\barb_k^T\Sigma^2\barb_k}\right)[(I-\barb_k\barb_k^T\Sigma)(I-\Sigma\barb_k\barb_k^T)]^\dagger\in\bR^{d\times d}$$
into the above equation gives
\begin{equation*}
\partial_{m,\barb_k} L_0(\bar{\btheta}) =\|\bar{\bb}_k\|^2(I-\bar{\bb}_k\bar{\bb}_k^T\Sigma)(I-\Sigma\bar{\bb}_k\bar{\bb}_k^T)\bE l'(f_0(\bx,\bar{\btheta}),y) a_k\sigma'(\bar{\bb}_k^T\bx)\bx.
\end{equation*}
\end{proof}

\section{The continuous formulation and the large width limit}
%\LY{maybe a different section title. Only 3.2 is about the large number limit.}

\subsection{The continuous formulation}
in this section, we consider the limiting model when the width of the two-layer neural networks
tends to infinity. In this case, a conventional way to represent the model is by integral
transformation~\cite{mei2018mean,chizat2018global,weinan2020machine}, and the parameters become a probability distribution on the parameter
space. Specifically, we consider the model
\begin{equation}\label{eqn:mf}
    f(\bx,\rho) = \int_{\bR\times\bR^d} a\sigma(\BN(\bb^T\bx))\rho(da, d\bb) = \int_{\bR\times\bR^d} a\sigma(\bar\bb^T\bx)\rho(da, d\bb),
\end{equation}
where $\rho$ is a probability distribution on $\bR\times\bS^{d-1}$. Note that this formulation can
also represent networks with finite width by taking $\rho$ as empirical distributions.

Still consider the data distribution $\bP$, the loss function is
\begin{equation*}
    L(\rho) = \bE_{(\bx,y)\sim\bP} l(f(\bx,\rho), y).
\end{equation*}
By~\eqref{eqn:gf_bn}, following the gradient flow, the velocity field of any particle $(a,\bb)$ is 
\begin{equation}
v_t(a,\bb) = -\left[\begin{array}{c}
     \bE l'(f(\bx,\rho), y)\sigma(\barb^T\bx)  \\
     \bE l'(f(\bx,\rho), y)a\sigma'(\bar\bb^T\bx)\frac{1}{\|\bb\|_\Sigma}\left(I-\frac{\Sigma\bb\bb^T}{\|\bb\|^2_\Sigma}\right)\bx
\end{array}\right],
\end{equation}
and thus $(\rho_t)_{t\geq0}$ as a time series of probability distributions satisfies the following continuity equation:
\begin{equation}\label{eqn:continuity_1}
    \partial_t \rho_t = -\divg(\rho_t v_t).
\end{equation}
Because of the batch normalization, any solution $(\rho_t)_{t\geq0}$ is always supported on $\bR\times\bS^{d-1}$. 

Similar to the finite width case, we want to write the above dynamics for the BN model as a
Wasserstein gradient flow on a Riemannian manifold of a non-BN model. To achieve this, we still
consider the manifold $\cM$ defined in Section~\ref{sec:manifold}. Let
$T:\bS^{d-1}\rightarrow\Omega$ be the ``normalization map'' to $\Omega$ defined as
$T\bb=\bar{\bb}=\frac{\bb}{\|\bb\|_\Sigma}$. Let $\bar{\rho}$ be the pushforward of $\rho$ by $\textrm{id}\times T$, defined as
\begin{equation}
    \bar\rho(A\times B) = \rho(A\times T^{-1}B)
\end{equation}
for any Borel measurable set $A\in\bR$ and $B\in\Omega$. Then, we easily have
\begin{equation*}
    f(\bx,\rho) = f_0(\bx,\bar{\rho}),
\end{equation*}
where $f_0(\bx,\bar{\rho})$ is the non-BN infinite-width neural network model with parameter distribution $\bar{\rho}$:
\begin{equation*}
    f_0(\bx,\bar{\rho}) = \int_{\cM} a\sigma(\barb^T\bx)d\bar{\rho}(a,\barb).
\end{equation*}
Note that the above integral is evaluated on $\cM$. Later when $\bar{\rho}$ is understood as a density function (e.g. in~\eqref{eqn:continuity_2}), the measure is evaluated by integrating the density function with the volume form on $\cM$. Now, let $(\bar{\rho}_t)_{t\geq0}$ be the series of distributions induced by
$(\rho_t)_{t\geq0}$ that satisfies $f(\bx,\rho_t) = f_0(\bx,\bar{\rho}_t)$ for any $t\geq0$. Then,
using the dynamics of $(a,\barb)$ derive in~\eqref{eqn:gf_bn_model}, $(\bar{\rho}_t)$ satisfies the
following continuity equation:
\begin{equation}\label{eqn:continuity_2}
  \partial_t \bar{\rho}_t = -\divg_m(\bar{\rho}_t  \bar{v}_t),
\end{equation}
%\LY{this should be a manifold divergence $\divg_m$, right?}
where $\divg_m$ is the divergence operator on the manifold, and the velocity field $\bar{v}_t$ is given by
\begin{equation}\label{eqn:vel_field}
    \bar{v}_t(a,\barb) = -\left[\begin{array}{c}
        \bE l'(f_0(\bx,\bar{\rho}_t),y)\sigma(\bar\bb^T\bx)   \\
        \|\bar{\bb}\|^2(I-\bar{\bb}\bar{\bb}^T\Sigma)(I-\Sigma\bar{\bb}\bar{\bb}^T)\bE l'(f_0(\bx,\bar{\rho}_t),y) a\sigma'(\bar\bb^T\bx)\bx
    \end{array}\right].
\end{equation}
Similar to the derivations in Section \ref{sec:manifold}, the velocity field above is the manifold
gradient of $L_0(\bar{\rho}_t):=\bE_{(\bx,y)\sim\bP}l(f_0(\bx,\bar{\rho}_t),y)$ (which is the same
as $L(\rho_t)$) on $\cM$, at the point $(a,\barb)$. Hence, the continuity
equation~\eqref{eqn:continuity_2} can be written as
\begin{equation}\label{eqn:gf_rho}
    \partial_t\bar{\rho}_t = \nabla_m\cdot\left(\bar{\rho}_t\nabla_m\frac{\delta
      L_0(\bar{\rho}_t)}{\delta \bar{\rho}_t}\right).
\end{equation}
%\LY{this should also be a manifold divergence $\nabla_m\cdot $, right?}
Let $\dist(\cdot,\cdot):\cM^2\rightarrow\bR_+$ be the metric function on $\cM$, and $\cP(\cM)$ be
the set of probability distributions on $\cM$ with finite second moment. The 2-Wasserstein distance
between any pair of $\bar{\rho}_1, \bar{\rho}_2\in\cP(\cM)$ is defined as
\begin{equation}
    W_2(\rho_1,\rho_2) = \left(\inf_{\gamma\in\Gamma(\rho_1,\rho_2)}\int_{\cM\times\cM}\dist(\bx_1,\bx_2)^2d\gamma(\bx_1,\bx_2)\right)^{\frac{1}{2}},
\end{equation}
where $\Gamma(\rho_1,\rho_2)$ contains all measures on $\cM\times\cM$ whose marginals with respect
to $\bx_1$ and $\bx_2$ are $\rho_1$ and $\rho_2$, respectively. Then, $\nabla_m\frac{\delta
  L_0(\bar{\rho}_t)}{\delta \bar{\rho}_t}$ is the Wasserstein differential of $L_0$ at
$\bar{\rho}_t$. Therefore, \eqref{eqn:gf_rho} is the Wasserstein gradient flow with respect to the
Wasserstein metric on the Riemannian manifold $\cM$. We summarize the result as the following
theorem:

\begin{theorem}\label{thm:wass}
Let $(\rho_t)_{t\geq0}$ be the trajectory of probability distributions obtained by learning the BN
model~\eqref{eqn:mf} (minimizing the loss $L(\rho)$) using gradient flow in the Euclidean space. Let
$(\bar{\rho}_t)_{t\geq0}$ be the ``normalized'' trajectories of $(\rho_t)$ supported on $\cM$. Then,
$(\bar{\rho}_t)$ satisfies the Wasserstein gradient flow of $L_0(\bar{\rho})$ on $\cM$.
\end{theorem}

\subsection{Convergence to the infinite-width limit}
In this part, we show that when the width tends to infinity, the empirical distribution given by the
finite parameters tends to a solution of~\eqref{eqn:gf_rho}. This also shows the existence of the
solution of~\eqref{eqn:gf_rho}. Specifically, a model with $m$ neurons with parameters on $\cM$ is
\begin{equation}
    \bar{f}(\bx,\bar{\rho}_m)=\int_\cM a\sigma(\barb^T\bx)d\bar{\rho}_m(a,\barb)=\frac{1}{m}\sum\limits_{k=1}^m a_k\sigma(\barb_k^T\bx),
\end{equation}
where $\bar{\rho}_m$ is the empirical distribution $\frac{1}{m}\sum\limits_{k=1}^m
\delta_{a_k}(a)\delta_{\barb_k}(\barb)$. Naturally, the dynamics of the $\bar{\rho}_{m,t}$ follows
the same continuity equation~\eqref{eqn:gf_rho} as the infinite width case,
i.e. $(\bar{\rho}_{m,t})$ solves the Wasserstein gradient flow of $L_0(\bar{\rho})$ starting from
$\bar{\rho}_{m,0}$.

Under appropriate assumptions, as $m\rightarrow\infty$, if $\bar{\rho}_{m,0}$ tends to some limiting distribution $\bar{\rho}_0$, then we can find a subsequence of trajectories $\bar{\rho}_{m,t}$ converging to the solution of~\eqref{eqn:gf_rho} initialized from $\bar{\rho}_0$. The result is stated in Theorem~\ref{thm:particle_limit}. First, we make the following assumptions:

\begin{assumption}\label{assump:1}
Assume:
\begin{itemize}
    \item There exists a constant $C_{\bx}$ such that for any input data $\bx$ we have $\|\bx\|\leq C_{\bx}$.
    \item The activation function is $L_\sigma$-Lipschitz.
    \item The derivative of the loss function, $l'$, is $L_{l'}$-Lipschitz.
\end{itemize}
\end{assumption}

\begin{theorem}\label{thm:particle_limit}
Let $(\bar{\rho}_{m,t})_{m=1}^\infty$ be the trajectories of empirical distributions generated by the gradient flow of parameters of models with different widths $m$. Let $\cM_r=[-r,r]\times\Omega$, and assume that any $\bar{\rho}_{m,0}$ is supported in $\cM_{r_0}$ for some $r_0>0$. If there exists $\bar{\rho}_0\in\cP(\cM)$ such that $\bar{\rho}_{m,0}\rightarrow\bar{\rho}_0$ weakly as $m\rightarrow\infty$, then there exists a subsequence of $(\bar{\rho}_{m,t})_{m=1}^\infty$, denoted still by $(\bar{\rho}_{m,t})_{m=1}^\infty$, and a trajectory $\bar{\rho}_t\in[0,\infty)\times\cP(\cM)$ which solves~\eqref{eqn:gf_rho} starting from $\bar{\rho}_0$, that satisfies $\bar{\rho}_{m,t}$ converges weakly to $\bar{\rho}_t$ for any $t>0$.
\end{theorem}

\begin{proof}
First, by the definition of $\cM$, we know that $\|\barb\|$ is always bounded. Let $C_\bb$ be an upper bound for $\bb$ that satisfies $\|\barb\|\leq C_\bb$. 

Our proof takes similar path like~\cite{chizat2018global}. For any $r>0$, let $t_r$ be the first time that some particles represented by some $\bar{\rho}_{m,t}$ goes out of $\cM_r$, i.e.
\begin{equation*}
    t_r:=\inf\{t>0, \exists m\in\bN, \bar{\rho}_{m,t}(\cM_r)<1\}. 
\end{equation*}
We first show that $t_r>0$ for any $r>r_0$, and $\lim\limits_{r\rightarrow\infty}t_r=\infty$.

To show $t_r>0$, note that the $\barb$ of any particle always moves on $\Omega$. Hence, we only need to consider $a$. For fixed $m$ and some $1\leq k\leq m$, let $(a_k,\barb_k)$ be the $k$-th particle of $\bar{\rho}_{m,t}$. Recall the dynamics of $a_k$:
\begin{equation*}
    \dot{a}_k = \bE_{\bx\sim\mu}l'(\bar{f}(\bx,\bar{\rho}_{m,t}))\sigma(\bar{\bb}_k^T\bx).
\end{equation*}
By Assumption~\ref{assump:1}, for any $t<t_r$, we have
\begin{align*}
|\dot{a}_k| &\leq \bE_{\bx\sim\mu}|l'(\bar{f}(\bx,\bar{\rho}_{m,t}))\sigma(\bar{\bb}_k^T\bx)| \\
 & \leq \bE_{\bx\sim\mu} (|l'(0)|+L_{l'}|\bar{f}(\bx,\bar{\rho}_{m,t})|)(|\sigma(0)|+L_\sigma C_\bb C_\bx) \\
 &\leq \left(|l'(0)|+L_{l'}\max_{\bx\sim\mu}|\bar{f}(\bx,\bar{\rho}_{m,t})|\right)\left(|\sigma(0)|+L_\sigma C_\bb C_\bx\right).
\end{align*}
Considering 
\begin{equation*}
    \bar{f}(\bx,\bar{\rho}_{m,t}) = \int a\sigma(\barb^T\bx)d\bar{\rho}_{m,t},
\end{equation*}
we have for any $\bx\sim\mu$, 
\begin{equation*}
    |\bar{f}(\bx,\bar{\rho}_{m,t})| \leq r\left(|\sigma(0)|+L_\sigma C_\bb C_\bx\right).
\end{equation*}
Therefore, come back to the dynamics of $a$, we have
\begin{equation*}
 |\dot{a}_k| \leq \left(|l'(0)|+L_{l'}r\left(|\sigma(0)|+L_\sigma C_\bb C_\bx\right)\right)\left(|\sigma(0)|+L_\sigma C_\bb C_\bx\right).
\end{equation*}
This mean there exist two constants $A,B$ that satisfies
\begin{equation*}
    |\dot{a}_k|\leq A+Br.
\end{equation*}
Consequently, as long as $\bar{\rho}_{m,t}$ is supported on $\cM_r$, the speed of $a_k$ is bounded. This indicates
\begin{equation}\label{eqn:particle_pf1}
    t_r \geq \frac{r-r_0}{A+Br} > 0.
\end{equation}

To show $\lim\limits_{r\rightarrow\infty} t_r =\infty$, we let $r_i=i\cdot r_0$ for $i\in\bN$. Then, by similar arguments as~\eqref{eqn:particle_pf1}, from $t_{r_i}$ to $t_{r_{i+1}}$, we have $|\dot{a}_k|\leq A+B(i+1)r_0$, and thus
\begin{equation*}
    t_{r_{i+1}} - t_{r_i} \geq \frac{(i+1)r_0-ir_0}{A+B(i+1)r_0} = \frac{r_0}{A+B(i+1)r_0}.
\end{equation*}
Hence, 
\begin{equation}
    \lim\limits_{r\rightarrow\infty} t_r = \lim\limits_{i\rightarrow\infty} t_{r_i} \geq \sum\limits_{i=1}^\infty \frac{r_0}{A+B(i+1)r_0} = \infty. 
\end{equation}

Next, we show the existence of convergent subsequence of $(\bar{\rho}_{m,t})$. We first show the result in $t\in[0,t_r]$ for finite $r$ by the Arzela-Ascoli theorem, i.e. we show that $(t\rightarrow \bar{\rho}_{m,t})$ is equicontinuous and pointwise bounded. 

For equicontinuity, we take the proof from~\cite{chizat2018global}. For any $0\leq t<t'\leq t_r$ and any $m\in\bN$, we have
\begin{align*}
W_2(\bar{\rho}_{m,t}, \bar{\rho}_{m,t'})^2 &\leq \frac{1}{m}\sum\limits_{k=1}^m \left(|a_k(t')-a_k(t)|^2+\|\barb_k(t')-\barb_k(t)\|^2\right) \\
  &\leq \frac{(t'-t)}{m} \int_t^{t'} \sum\limits_{k=1}^m \left(|\dot{a}_k|^2+\|\dot{\barb}_k\|^2\right)dt \\
  &= (t'-t) \left(\bar{L}(\bar{\rho}_{m,t'})-\bar{L}(\bar{\rho}_{m,t})\right)\\
  &\leq (t'-t)\bar{L}(\bar{\rho}_{m,t'}). 
\end{align*}
%\LY{should the second-to-last sign be an equality?}  
Since $\bar{\rho}_{m,t'}$ is supported on
$\cM_r$, the function represented by $\bar{\rho}_{m,t'}$, $f(\bx,\bar{\rho}_{m,t'})$ is bounded by
$r(\sigma(0)+L_\sigma C_\bb C_\bx)$. Hence, $\bar{L}(\bar{\rho}_{m,t'})$ is upper bounded by a
constant $C_r$ depending on $r$ but independent with $m$, i.e.,
\begin{equation*}
W_2(\bar{\rho}_{m,t}, \bar{\rho}_{m,t'}) \leq \sqrt{C_r(t'-t)}. 
\end{equation*}
This shows $(t\rightarrow \bar{\rho}_{m,t})$ is equicontinuous (in $W_2$ metric). 

On the other hand, pointwise boundedness follows naturally since all $a_k$ and $\barb_k$ are
uniformly bounded. Therefore, by the Arzela-Ascoli theorem, there exists a subsequence
$(t\rightarrow\bar{\rho}_{m_j, t})_{j=1}^\infty$ and a trajectory $\bar{\rho}_t$ such that
$\bar{\rho}_{m_j, t}\rightarrow\bar{\rho}_t$ weakly and uniformly for $t\in[0,t_r]$ as
$j\rightarrow\infty$.

Then, we extend the convergence from $[0,t_r]$ to $[0,\infty)$. Let $r_i=ir_0$. By the analysis
above, we can find a subsequence of $(\bar\rho)_{m,t}$ that converges weakly and uniformly in
$t\in[0,t_{2r_0}]$. Denote this subsequence by $(\bar{\rho}_{m_j^2,t})_{j=1}^\infty$. Then, still by
the same argument, we can find a subsequence in $(\bar{\rho}_{m_j^2,t})_{j=1}^\infty$, denoted by
$(\bar{\rho}_{m_j^3,t})_{j=1}^\infty$, that converges in $[0,t_{3r_0}]$. Repeat this process, for
any $i\geq2$, we can find a sequence $(\bar{\rho}_{m_j^i,t})_{j=1}^\infty$ that converges on
$[0,t_{ir_0}]$. Moreover, $(\bar{\rho}_{m_j^i,t})_{j=1}^\infty$ is a subsequence of
$(\bar{\rho}_{m_j^{i-1},t})_{j=1}^\infty$. Eventually, by the diagonal trick, it is easy to show
that the sequence $(\bar{\rho}_{m_j^{j+1},t})_{j=1}^\infty$
%\LY{should the superscript be $j$?}
converges weakly for any $t\in[0,\infty)$. 

Denote the limit obtained above by $\bar{\rho}_t$. For the last step of the proof, we show that $\bar{\rho}_t$ is a solution of the continuity equation~\eqref{eqn:gf_rho} starting from $\bar{\rho}_0$. Since the initial condition holds by definition, we only need to check $\bar{\rho}_t$ satisfies the equation weakly, i.e. for any $r>r_0$ and any bounded, Lipschitz test function $\phi(a,\barb,t)$ defined on $\cM\times[0,t_r]$ we have 
%\LY{manifold divergence below?}
\begin{equation*}
\int_0^{t_r}\int_{\cM} (\partial_t \phi + (\nabla_m\phi)^T G \bar{v}_t)d\bar{\rho}_tdt = 0,
\end{equation*}
%{\red In the equation above, gradient on manifold, and inner product on manifold.}
where $\bar{v}_t$ is the velocity field defined in~\eqref{eqn:vel_field} with $\bar{\rho}_t$, and $G$ is the metric matrix $G_{(a,\barb)}$. With an abuse of notation, we use $(\bar{\rho}_{m,t})_{m=1}^\infty$ to denote the subsequence that converges to $\bar{\rho}_t$. Let $v_{m,t}$ be the velocity field given by $\bar{\rho}_{m,t}$. Then, for any $m$, $\bar{\rho}_{m,t}$ satisfies the continuity equation, thus we have
\begin{equation*}
\int_0^{t_r}\int_{\cM} (\partial_t \phi + (\nabla_m\phi)^T G \bar{v}_{m,t})d\bar{\rho}_{m,t}dt = 0.
\end{equation*}
By the uniform convergence of $(\bar{\rho}_{m,t})$ on time, we have
\begin{equation*}
\int_0^{t_r}\int_{\cM}\partial_t \phi d\bar{\rho}_{m,t}dt \rightarrow \int_0^{t_r}\int_{\cM}\partial_t \phi d\bar{\rho}_{t}dt,\qquad m\rightarrow\infty.
\end{equation*}
Therefore, we only need to show 
\begin{equation}\label{eqn:particle_pf2}
\int_0^{t_r}\int_{\cM}(\nabla_m\phi)^T G \bar{v}_{m,t}d\bar{\rho}_{m,t}dt \rightarrow \int_0^{t_r}\int_{\cM}(\nabla_m\phi)^T G \bar{v}_{t}d\bar{\rho}_{t}dt.
\end{equation}

To proof~\eqref{eqn:particle_pf2}, first notice that for any bounded $\psi$ we have
\begin{align*}
&\left|\int_0^{t_r}\int_{\cM}\psi^T G \bar{v}_{m,t}d\bar{\rho}_{m,t}dt -\int_0^{t_r}\int_{\cM}\psi^T G \bar{v}_{t}d\bar{\rho}_{t}dt\right| \\
\leq & \left|\int_0^{t_r}\int_{\cM}\psi^T G(\bar{v}_{m,t}-\bar{v}_t)d\bar{\rho}_{m,t}dt\right| + \left|\int_0^{t_r}\int_{\cM}\psi^T G \bar{v}_t(d\bar{\rho}_{m,t}-d\bar{\rho}_t)dt\right| \\
:=& I+II.
\end{align*}
For $I$, since $\psi$ and $G$ are bounded, we show $\bar{v}_{m,t}\rightarrow \bar{v}_t$ uniformly for any $(a,\bb,t)$. Let $\bar{v}_{t,a}$ and $\bar{v}_{t,\barb}$ be the $a$ and $\barb$ component of the velocity field, respectively. $\bar{v}_{m,t,a}$ and $\bar{v}_{m,t,\barb}$ are similarly defined. Recall the definition of $\bar{v}$, for the $a$ component, we have
\begin{align*}
|\bar{v}_{m,t,a}-\bar{v}_{t,a}| & = \left| \bE_{\bx\sim\mu}(l'(\bar{f}(\bx,\bar{\rho}_{m,t}))-l'(\bar{f}(\bx,\bar{\rho}_{t})))\sigma(\barb^T\bx) \right| \\
  &\leq L_{l'}(\sigma(0)+L_\sigma C_{\bb}C_\bx)\bE_{\bx\sim\mu} |\bar{f}(\bx,\bar{\rho}_{m,t})-\bar{f}(\bx,\bar{\rho}_{t})| \\
  &\leq L_{l'}(\sigma(0)+L_\sigma C_{\bb}C_\bx)\bE_{\bx\sim\mu}\left|\int a\sigma(\barb^T\bx) (d\bar{\rho}_{m,t}-d\bar{\rho}_t)\right| \\
  &\leq L_{l'}(\sigma(0)+L_\sigma C_{\bb}C_\bx)^2r\|\bar{\rho}_{m,t}-\bar{\rho}_t\|_{\textrm{BL}},
\end{align*}
where the last inequality follows from the bound of $a\sigma(\barb^T\bx)$ by $r(\sigma(0)+L_\sigma C_{\bb}C_\bx)$. Therefore, by the uniform convergence of $\bar{\rho}_{m,t}$ to $\bar{\rho}_t$ for $t\in[0,t_r]$, we know $\bar{v}_{m,t,a}$ converges uniformly to $\bar{v}_{t,a}$. For the $\barb$ component, similarly we have
\begin{align*}
\|\bar{v}_{m,t,\barb}-\bar{v}_{t,\barb}\| &\leq \|\bar{\bb}\|^2\|(I-\bar{\bb}\bar{\bb}^T\Sigma)(I-\Sigma\bar{\bb}\bar{\bb}^T)\|\cdot\|\bE_{\bx\sim\mu}(l'(\bar{f}(\bx,\bar{\rho}_{m,t}))-l'(\bar{f}(\bx,\bar{\rho}_{t})))a\sigma'(\barb^T\bx)\bx\| \\
  &\leq C_{\barb}^2\|(I-\bar{\bb}\bar{\bb}^T\Sigma)(I-\Sigma\bar{\bb}\bar{\bb}^T)\|\cdot\|L_{l'}(\sigma(0)+L_\sigma C_{\bb}C_\bx)rL_\sigma C_{\bx} \|\bar{\rho}_{m,t}-\bar{\rho}_t\|_{\textrm{BL}}.
\end{align*}
Since both $\|\barb\barb^T\|$ and $\|\Sigma\|$ are upper bounded, we have $\|(I-\bar{\bb}\bar{\bb}^T\Sigma)(I-\Sigma\bar{\bb}\bar{\bb}^T)\|$ is upper bounded. Hence, the $\barb$ component of the velocity field converges uniformly, too. Combining the results above, we show $\bar{v}_{m,t}\rightarrow \bar{v}_t$ uniformly, and thus $I\rightarrow0$. 

For II, note that $\bar{v}_{t}$ is continuous and bounded. Therefore, there exists a constant $C$ such that 
\begin{equation*}
    II\leq C\int_0^{t_r} \|\bar{\rho}_{m,t}-\bar{\rho}_t\|_{\textrm{BL}}dt \rightarrow 0.
\end{equation*}

Combining the results for I and II, we finish the proof of~\eqref{eqn:particle_pf2}, and thus finish the whole proof. 
\end{proof}

The theorem above shows the existence of the solution for the continuity equation for any initial distribution with bounded support. If the solution is furthermore unique, then we can show that the sequence $\bar{\rho}_{m,t}$ converges to the unique solution. Here we did not establish uniqueness result. Hence, our theorem only shows the existence of convergent subsequence, and also the limit of any convergent subsequent is a solution of the continuity equation.

\subsection{Convergence of the Wasserstein gradient flow}

Next, we study the limit of the Wasserstein gradient flow as $t\rightarrow\infty$. We show a similar results as in~\cite{chizat2018global}, i.e., as long as $\bar{\rho}_t$ given by the Wasserstein gradient flow from $\bar{\rho}_0$ converges in $W_2$ to a distribution $\bar{\rho}_\infty$, we have $\bar{\rho}_\infty$ is the global minimum of $\bar{L}(\rho)$. Our proof follows the proof in Appendix C of~\cite{chizat2018global}, with several differences concerning the Riemannian manifold and the manifold gradient.

Recall that the loss function is $\bar{L}(\bar{\rho})=\bE_{\bx\sim\mu}l(\bar{f}(\bx,\bar\rho))$ and $\bar{\rho}_t$, $t\in[0,\infty)$ satisfies the continuity equation~\eqref{eqn:gf_rho}. Before stating the theorem, we make the following technical assumptions:

\begin{assumption}\label{assump:conv}
Assume
\begin{enumerate}
    \item The activation function $\sigma(\cdot)$ is differentiable and $\sigma'(\cdot)$ is $L_{\sigma'}$-Lipschitz continuous.
    \item The loss $l(\cdot)$ is convex and differentiable, and $l'(\cdot)$ is bounded on sublevel sets of $l$. (Recall that $l'$ is also $L_{l'}$-Lipschitz).
    \item For any $\rho$, the set of regular values of $g_{\rho}(\barb):=\bE_{\bx\sim\mu}l'(\bar{f}(\bx,\rho))\sigma(\barb^T\bx)$ as a function $\Omega\rightarrow\bR$ on the standard metric is dense in its range. 
\end{enumerate}
\end{assumption}

We also assume the initial distribution $\bar{\rho}_0$ is supported on some $\cM_{r_0}$ and satisfies the ``separation condition'' which is also used in~\cite{chizat2018global}: any curve connecting $\{r_0\}\times\Omega$ and $\{-r_0\}\times\Omega$ intersects with the support of $\bar{\rho}_0$. Under these assumptions, we have the following theorem.

\begin{theorem}\label{thm:conv}
Let $(\bar{\rho}_t)_{t\geq0}$ be a solution of the continuity equation~\eqref{eqn:gf_rho}. Assume Assumption~\ref{assump:conv} holds. Assume there exists $r_0>0$ such that $\bar{\rho}_0$ is supported on $\cM_{r_0}$, and separates $\{r_0\}\times\Omega$ and $\{-r_0\}\times\Omega$. If there exists $\bar{\rho}_\infty$ such that
\begin{equation*}
    \lim\limits_{t\rightarrow\infty} W_2(\bar{\rho}_t, \bar{\rho}_\infty) = 0,
\end{equation*}
then $\bar{\rho}_\infty$ is a global minimizer of $\bar{L}(\rho)$ over all probability distributions on $\cM$. 
\end{theorem}

\begin{proof}
Our model falls into the ``partial 1-homogeneous'' case in~\cite{chizat2018global}. Hence, Theorem~\ref{thm:conv} holds mostly following the proof therein. We only need to show that the regular value theorem still holds if the $g_{\rho}$ in Assumption~\ref{assump:conv} is treated as a function on $\Omega$ with metric $G_{\barb}$. Equivalently, we show a regular value of $g_{\rho}$ as a function under standard metric is still a regular value $g_{\rho}$ as a function under the metric $G_{\barb}$. Let $\nabla g_{\rho}(\barb)$ be the nonzero gradient of $g_{\rho}$ at some regular point $\barb$. Since $\nabla g_{\rho}(\barb)$ is tangent to $\Omega$, we have $\barb^T \Sigma \nabla g_{\rho}(\barb)=0$. Therefore,
\begin{align*}
    \nabla g_{\rho}(\barb)^T(I-\Sigma\barb\barb^T)\nabla g_{\rho}(\barb) & = \nabla g_{\rho}(\barb)^T\nabla g_{\rho}(\barb)- \nabla g_{\rho}(\barb)^T\Sigma\barb \barb^T\nabla g_{\rho}(\barb) \\
    & = \nabla g_{\rho}(\barb)^T\nabla g_{\rho}(\barb) \\
    &>0.
\end{align*}
This means $(I-\Sigma\barb\barb^T)\nabla g_{\rho}(\barb)$ is nonzero. Hence,
\begin{align*}
\nabla g_{\rho}(\barb)^T\nabla_m g_{\rho}(\barb)^T &= \nabla g_{\rho}(\barb)^T \|\barb\|^2(I-\barb\barb^T\Sigma)(I-\Sigma\barb\barb^T)\nabla g_{\rho}(\barb)^T \\
&= \|\barb\|^2 \|(I-\Sigma\barb\barb^T)\nabla g_{\rho}(\barb)^T\|^2\\
&>0.
\end{align*}
This shows $\nabla_m g_{\rho}(\barb)\neq0$, which means $\barb$ is also a regular point for $g_{\rho}$ under metric $G_{\barb}$. 

\end{proof}

\section{Discussion}
In this section, we discuss the potential benefits brought by batch normalization to the training dynamics of neural networks. We compare the dynamics of the models with BN
\begin{align}
  \dot{a} &= -\bE_{\bx\sim\mu}l'(f(\bx,\rho))\sigma(\bar\bb^T\bx) \nonumber \\
  \dot{\bb} &= -\bE_{\bx\sim\mu}l'(f(\bx,\rho)) a\sigma'(\bar\bb^T\bx)\frac{1}{\|\bb\|_\Sigma}\left(I-\frac{\Sigma\bb\bb^T}{\|\bb\|^2_\Sigma}\right)\bx, \label{eqn:dyn_bn}
\end{align}
and without BN
\begin{align}
  \dot{a} &= -\bE_{\bx\sim\mu}l'(f(\bx,\rho))\sigma(\bb^T\bx) \nonumber \\
  \dot{\bb} &= -\bE_{\bx\sim\mu}l'(f(\bx,\rho)) a\sigma'(\bb^T\bx)\bx. \label{eqn:dyn_nobn}
\end{align}
We focus on the speed of parameters $\bb$, i.e., the speed of the changing of the neurons' directions.

\subsection{Speed and data distribution}
By~\eqref{eqn:dyn_bn}, for the model with batch normalization the speed of $\bb$ depends on
$\|\bb\|_\Sigma$, which further depends on the distribution of the input data. Thus, when the data
distribution is not isotropic, the speed of the neuron will be influenced by the relation between its direction and the data distribution. Specifically, when $\bb$ points to the direction in which $\bx$ is
small, i.e. $\bE\bb^T\bx$ is small, its moving speed will get bigger because $\|\bb\|_\Sigma$ is
small. On the other hand, when $\bb$ points to the direction in which $\bx$ is big, the moving speed
will get smaller. This anisotropic speed effect helps neurons escape the non-significant
directions where data are small and concentrate to significant directions where data are large, thus speeds up the
learning of the features along these significant directions. In figure~\ref{fig:data_dist}
we show this effect with a comparison with vanilla models.

\begin{figure}
    \centering
    \includegraphics[width=0.175\textwidth]{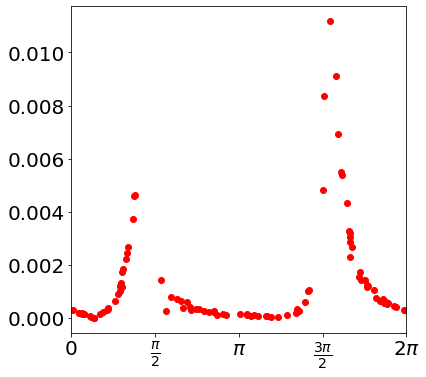}
    \hspace{-3mm}
    \includegraphics[width=0.16\textwidth]{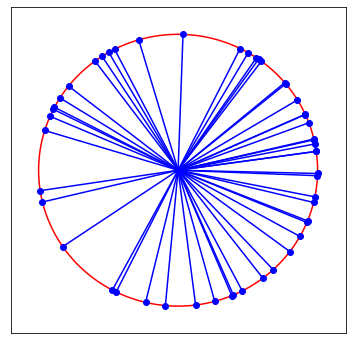}
    \hspace{-2mm}
    \includegraphics[width=0.16\textwidth]{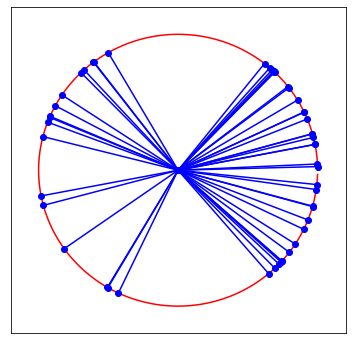}
    \hspace{-2mm}
    \includegraphics[width=0.16\textwidth]{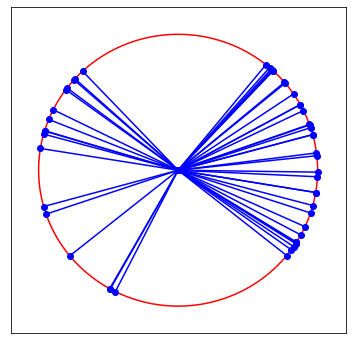}
    \hspace{-2mm}
    \includegraphics[width=0.16\textwidth]{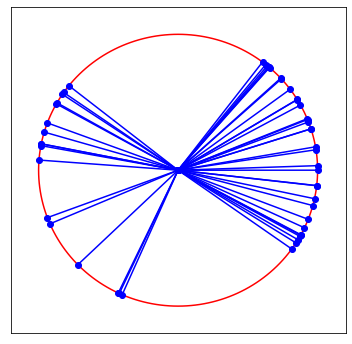}
    \hspace{-2mm}
    \includegraphics[width=0.16\textwidth]{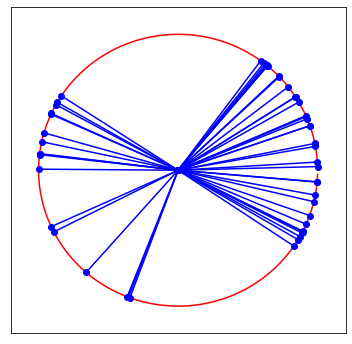}
    
    \includegraphics[width=0.175\textwidth]{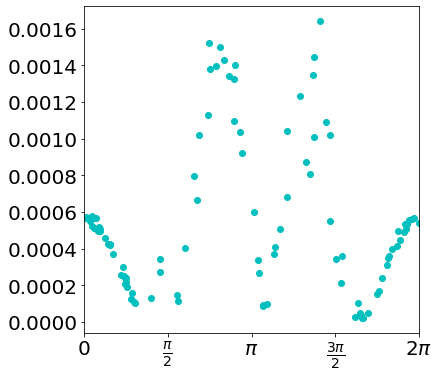}
    \hspace{-3mm}
    \includegraphics[width=0.16\textwidth]{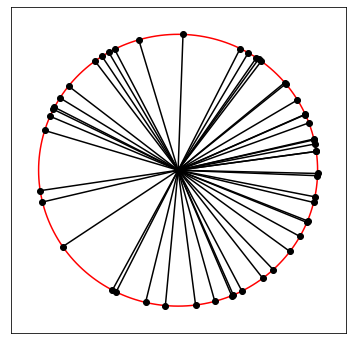}
    \hspace{-2mm}
    \includegraphics[width=0.16\textwidth]{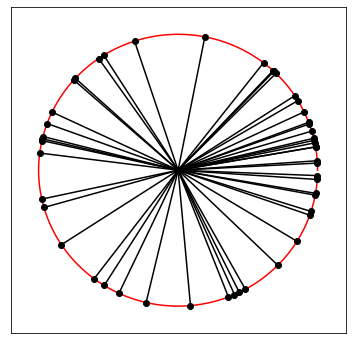}
    \hspace{-2mm}
    \includegraphics[width=0.16\textwidth]{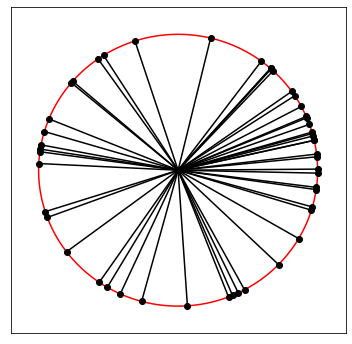}
    \hspace{-2mm}
    \includegraphics[width=0.16\textwidth]{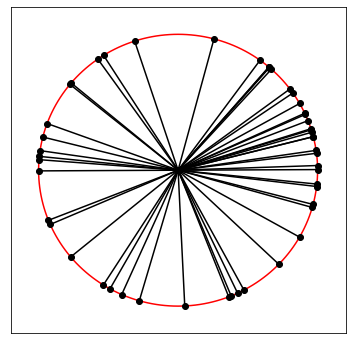}
    \hspace{-2mm}
    \includegraphics[width=0.16\textwidth]{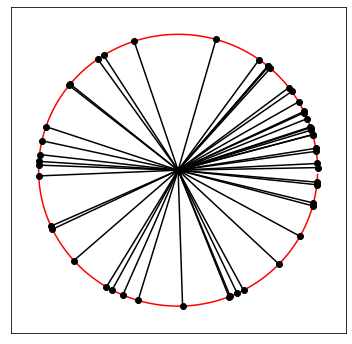}
    
    \caption{The speed of neurons and distribution of neurons at different time of models with and without BN. The data $\bx\in\bR^2$ with $\Sigma=\left[\begin{array}{cc}
        5 & 0 \\
        0 & 1
        \end{array}\right]$.
      The leftmost column shows the speed of the neuron ($\bb$ along the tangent direction of the
      unit circle) with respect to directions. The speed is calculated at the 500-th iteration. The
      top panel is for BN model, while the bottom panel is for non-BN model. Note that the scale of
      the top panel's vertical axis is one order of magnitude bigger than that of the bottom panel.
      The other five columns show the directions of the neurons at different training times (0,
      2000, 4000, 6000, 8000 iterations from left to right). The first row shows results for BN
      model, and the second row shows results for non-BN model. The red circle is the unit
      circle. The $\bb$'s for non-BN model are normalized to the unit circle.}
    \label{fig:data_dist}
\end{figure}

\subsection{Speed and parameter magnitude}
The existence of the term $\frac{1}{\|\bb\|_\Sigma}$ also produces a connection between the speed of the neuron with its magnitude. To see this, consider two neurons $(a,\bb_1)$ and $(a,\bb_2)$ with $\bb_2=c\bb_1$ for a positive constant $c$. Assume $\sigma'(\cdot)$ satisfies $\sigma'(cw)=\sigma'(w)$ for any input $w\in\bR$ and $c>0$. (This is true is $\sigma$ is ReLU or leaky ReLU). Then, for models without BN, we have
\begin{equation*}
    \dot{\bb}_2 = \dot{\bb}_1,
\end{equation*}
while for models with BN, we have
\begin{equation*}
    \dot{\bb}_2 = \frac{1}{c}\dot{\bb}_1.
\end{equation*}
Hence, for the BN model, smaller neurons move faster. This is verified numerically in Figure~\ref{fig:param_mag}. 

Considering that for the BN model neurons with different magnitude (of $\bb$) actually express the same function, the influence of neuron magnitude on its speed allows the model to explore different learning rates--those neurons with small $\bb$ are learning with big learning rates while the neurons with large $\bb$ learn with small learning rates. This effect gives the BN model an ``adaptivity'' to select the right learning rate itself, and hence less sensitivity to the choice of learning rate.

\begin{figure}
    \centering
    \includegraphics[width=0.19\textwidth]{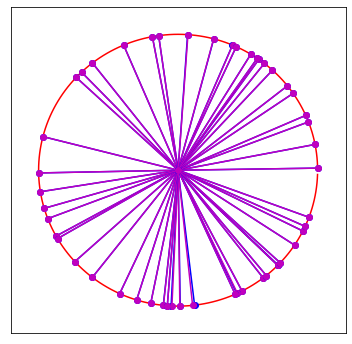}
    \hspace{-2mm}
    \includegraphics[width=0.19\textwidth]{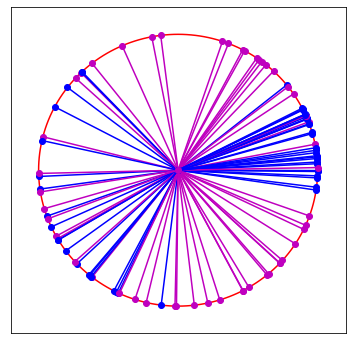}
    \hspace{-2mm}
    \includegraphics[width=0.19\textwidth]{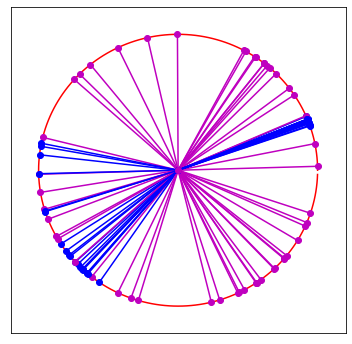}
    \hspace{-2mm}
    \includegraphics[width=0.19\textwidth]{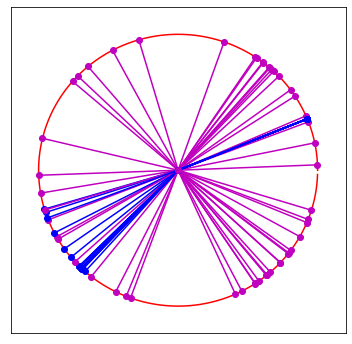}
    \hspace{-2mm}
    \includegraphics[width=0.19\textwidth]{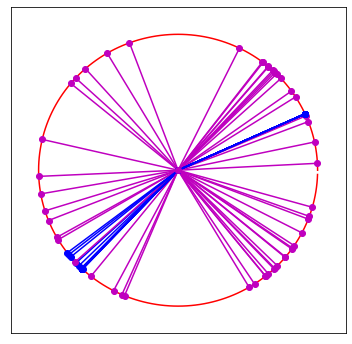}

    \caption{The directions of the neurons for a model with multiscale initialization. Neurons shown in purple lines have $10\times$ bigger initialization than neurons shown in blue lines. From left to right, the figures show neuron directions at iteration 0, 2000, 4000, 6000, 8000, 10000. The figures show that blue neurons converge faster than purple neurons, i.e. shorter neurons moves faster than long neurons.}
    \label{fig:param_mag}
\end{figure}

\subsection{The ``first-step amplification''}
By the above discussion, for the BN model, neurons with smaller magnitude move faster. At a first glance, this speed effect may cause problem when very small neurons exist, whose speeds are very large. However, for the discrete dynamics, i.e. gradient descent algorithm, this would not be a problem, because very small neurons will get big after the first iteration. This phenomenon is illustrated in Figure~\ref{fig:first_step}. At the first iteration, the gradient for a small $\bb$ is very large. Moreover, by the nature of batch normalization, the gradient lies in the tangent space of the sphere with radius $\|\bb\|$. Hence, the first iteration will only make the magnitude of $\|\bb\|$ bigger. The smaller the neuron is initially, the bigger it becomes after one iteration. Therefore, starting from the second iteration, all neurons are large enough to avoid unstable dynamics. 

\begin{figure}
    \centering
    \includegraphics[width=0.65\textwidth]{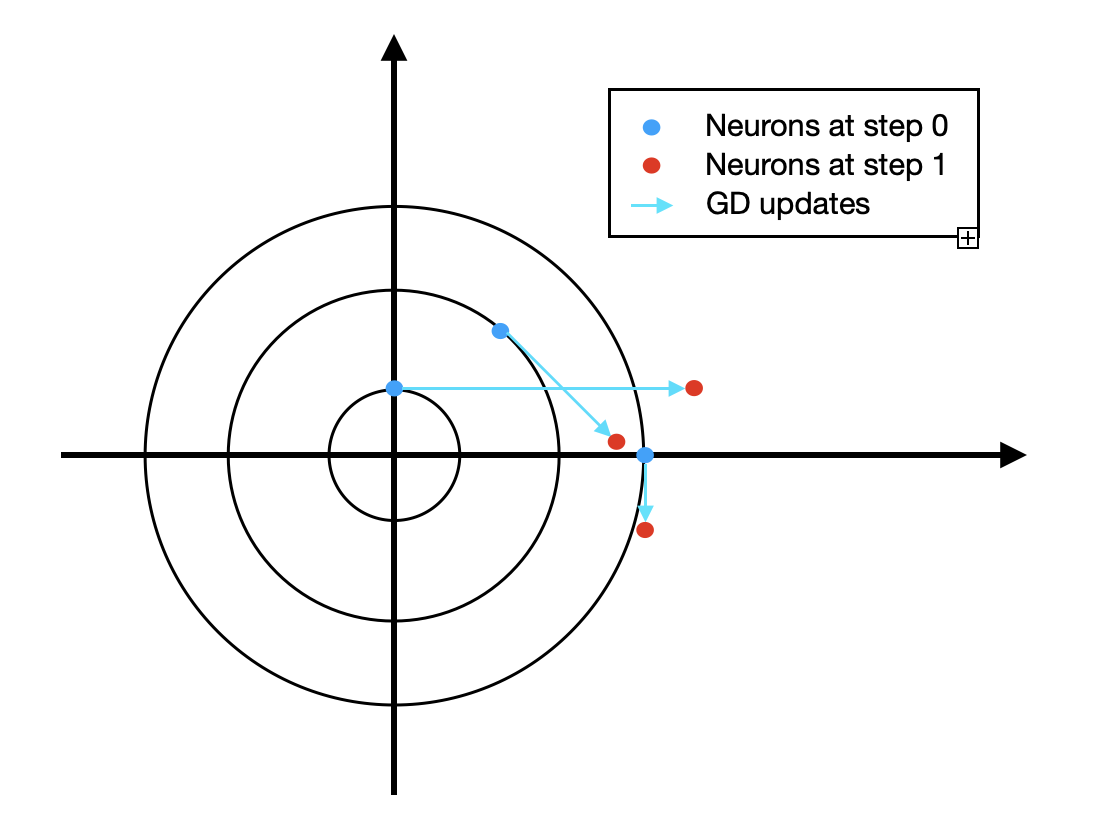}
    \caption{The first-step amplification effect: the gradients of small neurons are larger than that of big neurons. Hence, after the first iteration, very small neurons become large and thus will not suffer from unstable time-discrete dynamics.}
    \label{fig:first_step}
\end{figure}

\section{Related work}

The good performance of batch normalization was initially believed to be caused by the prevention of internal covariance shift~\cite{ioffe2015batch}. Later work challenged this point of view by numerical observations~\cite{santurkar2018does}, and connected the benefit of BN with optimization landscape smoothness. Many other attempts are made to understand batch normalization. In~\cite{bjorck2018understanding}, it is observed that BN can enable larger learning rate, which leads to faster convergence. In~\cite{wei2019mean}, it is shown that BN can flatten the optimization landscape. The authors of~\cite{luo2018towards} analyzed the regularization effect of BN. A Riemannian manifold understanding was provided in~\cite{cho2017riemannian}, in which the authors proposed a manifold optimization framework based on the scale invariant property of BN. However, in~\cite{cho2017riemannian} a Grassmann manifold with standard metric is considered, thus normalization is still needed on the manifold. As a comparison, in this paper we find a special metric to eliminate the normalization step. In practice, a series of other normalization methods are proposed as substitutes of batch normalization~\cite{ba2016layer,salimans2016weight,klambauer2017self,wu2018group}.

On the other hand, the mean-field formulation for neural networks is a helpful tool to analyze and understand the training dynamics of neural networks, especially those with infinite neurons. The mean-field formulation was first studied for two-layer neural networks~\cite{mei2018mean,chizat2018global,rotskoff2018neural}, and then extended to deep networks~\cite{araujo2019mean,nguyen2019mean,fang2021modeling}.

\section{Conclusion}
In this paper, we study the infinite-width limit of two-layer neural networks with batch normalization. We show that, the dynamics of the model with batch normalization is equivalent with the dynamics of the original model on a Riemannian manifold. Then, we derive a mean-field formulation for the training dynamics with GD. The mean-field dynamics is a Wasserstein gradient flow on a Riemannian manifold. Based on existing results, we prove the existence of the Wasserstein gradient flow, and show that once it converges, the limit is a global minimum. The Riemannian manifold understanding for batch normalization model provides us a new perspective to study the benefit of batch normalization.

{
\bibliographystyle{plain}
\bibliography{bn}

\begin{thebibliography}{10}

\bibitem{araujo2019mean}
Dyego Ara{\'u}jo, Roberto~I Oliveira, and Daniel Yukimura.
\newblock A mean-field limit for certain deep neural networks.
\newblock {\em arXiv preprint arXiv:1906.00193}, 2019.

\bibitem{arora2018theoretical}
Sanjeev Arora, Zhiyuan Li, and Kaifeng Lyu.
\newblock Theoretical analysis of auto rate-tuning by batch normalization.
\newblock {\em arXiv preprint arXiv:1812.03981}, 2018.

\bibitem{ba2016layer}
Jimmy~Lei Ba, Jamie~Ryan Kiros, and Geoffrey~E Hinton.
\newblock Layer normalization.
\newblock {\em arXiv preprint arXiv:1607.06450}, 2016.

\bibitem{bjorck2018understanding}
Johan Bjorck, Carla Gomes, Bart Selman, and Kilian~Q Weinberger.
\newblock Understanding batch normalization.
\newblock {\em arXiv preprint arXiv:1806.02375}, 2018.

\bibitem{chizat2018global}
Lenaic Chizat and Francis Bach.
\newblock On the global convergence of gradient descent for over-parameterized
  models using optimal transport.
\newblock {\em arXiv preprint arXiv:1805.09545}, 2018.

\bibitem{cho2017riemannian}
Minhyung Cho and Jaehyung Lee.
\newblock Riemannian approach to batch normalization.
\newblock {\em arXiv preprint arXiv:1709.09603}, 2017.

\bibitem{fang2021modeling}
Cong Fang, Jason Lee, Pengkun Yang, and Tong Zhang.
\newblock Modeling from features: a mean-field framework for over-parameterized
  deep neural networks.
\newblock In {\em Conference on Learning Theory}, pages 1887--1936. PMLR, 2021.

\bibitem{ioffe2015batch}
Sergey Ioffe and Christian Szegedy.
\newblock Batch normalization: Accelerating deep network training by reducing
  internal covariate shift.
\newblock In {\em International conference on machine learning}, pages
  448--456. PMLR, 2015.

\bibitem{klambauer2017self}
G{\"u}nter Klambauer, Thomas Unterthiner, Andreas Mayr, and Sepp Hochreiter.
\newblock Self-normalizing neural networks.
\newblock In {\em Proceedings of the 31st international conference on neural
  information processing systems}, pages 972--981, 2017.

\bibitem{kohler2018towards}
Jonas Kohler, Hadi Daneshmand, Aurelien Lucchi, Ming Zhou, Klaus Neymeyr, and
  Thomas Hofmann.
\newblock Towards a theoretical understanding of batch normalization.
\newblock {\em stat}, 1050:27, 2018.

\bibitem{luo2018towards}
Ping Luo, Xinjiang Wang, Wenqi Shao, and Zhanglin Peng.
\newblock Towards understanding regularization in batch normalization.
\newblock {\em arXiv preprint arXiv:1809.00846}, 2018.

\bibitem{mei2018mean}
Song Mei, Andrea Montanari, and Phan-Minh Nguyen.
\newblock A mean field view of the landscape of two-layer neural networks.
\newblock {\em Proceedings of the National Academy of Sciences},
  115(33):E7665--E7671, 2018.

\bibitem{nguyen2019mean}
Phan-Minh Nguyen.
\newblock Mean field limit of the learning dynamics of multilayer neural
  networks.
\newblock {\em arXiv preprint arXiv:1902.02880}, 2019.

\bibitem{rotskoff2018neural}
Grant~M Rotskoff and Eric Vanden-Eijnden.
\newblock Neural networks as interacting particle systems: Asymptotic convexity
  of the loss landscape and universal scaling of the approximation error.
\newblock {\em stat}, 1050:22, 2018.

\bibitem{salimans2016weight}
Tim Salimans and Durk~P Kingma.
\newblock Weight normalization: A simple reparameterization to accelerate
  training of deep neural networks.
\newblock {\em Advances in neural information processing systems}, 29:901--909,
  2016.

\bibitem{santurkar2018does}
Shibani Santurkar, Dimitris Tsipras, Andrew Ilyas, and Aleksander Madry.
\newblock How does batch normalization help optimization?
\newblock In {\em Proceedings of the 32nd international conference on neural
  information processing systems}, pages 2488--2498, 2018.

\bibitem{wei2019mean}
Mingwei Wei, James Stokes, and David~J Schwab.
\newblock Mean-field analysis of batch normalization.
\newblock {\em arXiv preprint arXiv:1903.02606}, 2019.

\bibitem{weinan2020machine}
E~Weinan, Chao Ma, and Lei Wu.
\newblock Machine learning from a continuous viewpoint, i.
\newblock {\em Science China Mathematics}, 63(11):2233--2266, 2020.

\bibitem{wu2018group}
Yuxin Wu and Kaiming He.
\newblock Group normalization.
\newblock In {\em Proceedings of the European conference on computer vision
  (ECCV)}, pages 3--19, 2018.

\end{thebibliography}
}

\end{document}